\documentclass[runningheads]{llncs}
\usepackage{graphicx}
\usepackage{hyperref}
\usepackage{times}
\usepackage{xcolor}
\usepackage{amsmath,amssymb,bbm}
\usepackage{enumitem}
\usepackage{booktabs}  % professional-quality tables
\usepackage[nameinlink]{cleveref}
\usepackage{algorithm}
\usepackage{algpseudocode}

\usepackage{mathrsfs}

\newcommand{\R}{\mathbb R}

\newcommand{\E}{\mathbb E}
\newcommand{\Pbb}{\mathbb P}
\newcommand{\one}{\mathbbm{1}}

\newcommand{\Fcal}{\mathcal{F}}
\newcommand{\Rcal}{\mathscr{R}}

\newcommand{\Lcal}{\mathcal{L}}
\newcommand{\Var}{\mathbb{V}}

\newcommand{\blind}{1}%0=blind, 1=un-blind

\begin{document}
\title{Softmax gradient policy for variance minimization and risk-averse multi armed bandits}
\titlerunning{Variance minimization bandits}
\if1\blind
{%blind version
\author{
Gabriel Turinici\inst{}\orcidID{0000-0003-2713-006X} 
}%
\authorrunning{G. Turinici}
\institute{CEREMADE, \\  Universit\'e Paris Dauphine - PSL, CNRS, Paris, France \\
		\email{gabriel.turinici@dauphine.fr}\\
		\url{https://turinici.com} }
}\fi%end non blind version
\if0\blind{
\author{anonymized version\institute{anonymous institute}}
}\fi

\maketitle              % typeset the header of the contribution
\begin{abstract}
Algorithms for the Multi-Armed Bandit (MAB) problem play a central role in sequential decision-making and have been extensively explored both theoretically and numerically. While most classical approaches aim to identify the arm with the highest expected reward, we focus on a risk-aware setting where the goal is to select the arm with the lowest variance, favoring stability over potentially high but uncertain returns. To model the decision process, we consider a softmax parameterization of the policy; we  
propose a new algorithm to select the minimal variance (or minimal risk) arm and prove its convergence under natural conditions. 
The algorithm constructs an unbiased estimate of the objective by using two independent draws from the selected arm's distribution.
We provide numerical experiments that illustrate the practical behavior of these algorithms and offer guidance on implementation choices. The setting also covers general risk-aware problems where there is a trade-off between maximizing the average reward and minimizing its variance. 
\end{abstract}
\section{Introduction}

Decision-making under uncertainty is a fundamental challenge in many domains, including finance, healthcare, and online recommendation systems. The Multi-Armed Bandit (MAB) framework \cite{contextual_mab,slivkins_introduction_mab_2019} provides a principled approach to balancing exploration and exploitation in such settings. Traditionally, MAB algorithms have focused on maximizing expected rewards, which is appropriate when risk considerations are secondary. However, in many real-world applications, the variability of outcomes can be as critical as their average performance, especially in environments where stability and reliability are essential. This has motivated the development of risk-aware bandit models, where the objective incorporates measures of uncertainty or risk rather than focusing solely on the mean.
Our work adopts this perspective by investigating strategies that prioritize arms with minimal variance, reflecting a preference for consistent outcomes over potentially high but volatile returns. This approach aligns with broader trends and advancements in reinforcement learning
\cite{mnih_human_level_rl_2015,silver_rl_mastering_go_2016,openai2024gpt4,rl_recommender22,bojarski2016end_rl_cars,RL_medecine1,RL_medecine2}, particularly distributional reinforcement learning~\cite{bellemare2023distributional,bellemare2017distributional,dabney2017qr,dabney2018iqn,rowland2018analysis,wiltzer2024multivariate,antonetti2025gaussian}, which emphasizes modeling the entire distribution of returns rather than just their expectation. By considering the distributional properties of rewards, such methods enable richer decision-making criteria that account for risk, variability, and tail behavior. In this study, we explore how these ideas can be adapted to the bandit setting, using a softmax parameterization \cite[section 2.8 and chap. 13]{sutton_reinforcement_2018} to guide arm selection in a smooth and theoretically grounded manner.

The paper is organized as follows: we present a short literature review in the remainder of this section and then give the main framework in Section \ref{sec:notations}. We give theoretical results in Section \ref{sec:cv_proof}  and empirical ones in Section~\ref{sec:numerical}. We conclude with \Cref{sec:conclusion}.

\subsection{Brief literature review}

The concept of \emph{distributional reinforcement learning (DRL)} emerged as an extension of classical RL, which traditionally models the expected return as a scalar value. Instead, DRL focuses on learning the entire distribution of returns, enabling richer decision-making and risk-sensitive policies. The seminal work by \cite{bellemare2017distributional} introduced the distributional perspective through the Categorical DQN (C51) algorithm, which approximates the return distribution using a categorical representation. This approach demonstrated improved performance and stability compared to standard value-based methods. Building on this, \cite{dabney2017qr} proposed Quantile Regression DQN (QR-DQN), replacing categorical approximations with quantile-based modeling, which offers greater flexibility and accuracy. Later, \cite{dabney2018iqn} introduced Implicit Quantile Networks (IQN), a method that learns the full quantile function implicitly, allowing for dynamic risk-sensitive policies and better approximation of the return distribution.

The theoretical foundations of these methods were further analyzed by \cite{rowland2018analysis}, who examined the convergence properties of C51 and the role of the projection operator. More recent work has extended DRL beyond scalar rewards: \cite{wiltzer2024multivariate} developed a framework for multivariate distributional RL, addressing settings with vector-valued rewards, while \cite{antonetti2025gaussian} explored modeling return distributions using Gaussian mixtures, providing new convergence guarantees and improved flexibility. For a comprehensive treatment of the field see the textbook by Bellemare et al. \cite{bellemare2023distributional}.

\subsubsection{Risk-Aware Multi-Armed Bandits}

While distributional reinforcement learning (DRL) focuses on modeling the full distribution of returns in sequential decision-making, similar concerns arise in the simpler setting of Multi-Armed Bandits (MAB). Classical MAB algorithms aim to maximize expected reward, which is appropriate when risk considerations are negligible. However, in many real-world applications such as finance or safety-critical systems, variability and tail risk are critical factors. This has motivated the development of risk-aware bandit models, where the objective incorporates measures beyond the mean, such as variance, mean-variance trade-offs, quantiles, or Conditional Value-at-Risk (CVaR). Recent works have explored these risk-sensitive criteria and even distributional approaches in bandits, reflecting a broader trend toward algorithms that account for uncertainty and stability in decision-making.

Early work by \cite{sani2012risk} introduced risk-aversion in bandits by considering the mean-variance criterion $\text{Variance} - \lambda \cdot \text{Mean}$ and proposed algorithms such as MV-LCB and ExpExp, which adapt UCB-style confidence bounds to account for risk. Building on this, \cite{vakili2016risk} provided a refined analysis of these algorithms and extended the framework to include both mean-variance and Value-at-Risk objectives \cite{vakili2015mean}. In a more general setting, \cite{zimin2014generalized} considered objectives of the form $f(\mu,\sigma^2)$ and proposed the $\phi$-LCB algorithm, which leverages Chernoff-Hoeffding bounds to achieve desirable performance guarantees. Another line of work by \cite{galichet2013exploration} focused on safety by introducing MARAB, an algorithm that uses CVaR as the arm quality measure to limit exploration of risky arms. More recently, \cite{huo2017risk} studied risk-aware bandits in the context of portfolio selection, again relying on UCB-like strategies to balance exploration and risk. Finally, \cite{varianceoptimal2025} explored variance-optimal arm selection, focusing on regret minimization and best-arm identification when the objective is to select the arm with the highest variance, using UCB-based approaches.

These contributions demonstrate that most risk-aware bandit algorithms rely on UCB-style confidence bounds or LCB variants to incorporate risk measures. In contrast, our work adopts a different perspective: we consider a \emph{softmax parameterization} of the policy in a risk-aware setting, specifically targeting arms with minimal variance. To the best of our knowledge, this approach has not been explored in the existing literature and provides a novel way to integrate risk-awareness into bandit algorithms. Furthermore a fundamental difference with other algorithms is that in our vision we do not keep variance related statistics but compute it on the fly without relying on an estimate for the average, at it will be described in \Cref{alg:variance-mab} below.

\subsubsection{Convergence of MABs}

Recent work by Mei et al.~\cite{mei_global_2020,mei2023stochastic} established that exact-gradient policy gradient methods with softmax parameterization converge at rate $O(1/t)$, and that entropy regularization can accelerate convergence. Surprisingly, they also show that stochastic gradients can converge under constant learning rates, albeit requiring many iterations (see also \cite{baudry2025does}). Their analysis builds on classical stochastic approximation theory initiated by Robbins and Monro~\cite{robbins_stochastic_1951}, with comprehensive treatments in \cite{chen_stochastic_2002} and recent extensions to non-convex settings~\cite{sgd_conv_non_cx20,mertikopoulos_almost_2020}. Additional results address biased gradients~\cite{ajalloeian2021convergencesgdbiasedgradients}, convex objectives~\cite{regularizedSGD_convex}, and entropic regularization~\cite{weissmann2025almost}. These studies show that decaying regularization preserves asymptotic behavior while improving early stability. However, applying these results to MAB requires care: assumptions such as uniqueness of critical points, boundedness, and Lyapunov conditions~\cite{mertikopoulos_almost_2020} often fail in diverse settings. Despite these challenges, our approach adapts this theoretical foundation to variance minimizing and risk-aware bandits, combining estimates and convergence arguments not previously applied in this context.

\section{Softmax parameterized policy gradient Multi Armed Bandit} \label{sec:notations}

We consider the Multi-Armed Bandit (MAB) problem in the sense of \cite{sutton_reinforcement_2018}, where a learner repeatedly chooses among $k$ alternatives, called \emph{arms}, indexed by $a \in \{1,2,\dots,k\}$. Each arm $a$ is associated with a reward distribution, from which rewards are sampled when the arm is selected. We denote by $\mu_a$ the mean reward and by $\sigma_a^2$ the variance of the reward distribution for arm $a$. Thus, if $R(a)$ is the random reward from arm $a$, then:
\begin{align}
\mathbb{E}[R(a)] &= \mu_a, \\
\Var[R(a)] &= \sigma_a^2,
\end{align}
where $\E[\cdot]$ and $\Var[\cdot]$ are the average and variance operators respectively.

Classical formulations often assume $\sigma_a^2 = 1$ for all arms, but here we allow heterogeneous variances, which is essential for risk-aware decision-making. In the classical setting, at time step $t$, the learner selects an arm $A_t$ sampled using some policy $\Pi_t \in \{ \pi \in \R^k : \sum_a \pi_a = 1, \pi_a \ge 0\}$ and observes a reward $R_t$ drawn from the corresponding distribution, which we will denote $R_t \sim R(A_t)$.

The policy $\Pi$ is parameterized by a preference vector $H \in \mathbb{R}^k$  through the \emph{softmax} function:
\begin{equation}
\Pi_H(a) = \frac{e^{H(a)}}{\sum_{b=1}^k e^{H(b)}}.
\label{eq:softmax_policy}
\end{equation}

\subsection{Variance minimizing Multi-Armed Bandit}

While the standard objective is to maximize cumulative expected reward, our risk-aware setting aims to favor arms with lower variability. Specifically, we  proceed as follows. 
We begin by defining the target functional in terms of the reward distribution under a policy $\Pi$:
\begin{align}
J(\Pi) = \E_{A \sim \Pi} [\Var[R(A)|A ]],
\label{eq:functional_pi_0}
\end{align}

Under the softmax parameterization, the optimization problem becomes:
\begin{equation}
\text{minimize}_{H \in \mathbb{R}^k} \; \mathcal{L}(H),
\label{eq:problem_mab_as_minimization}
\end{equation}
where $\mathcal{L}(H)$ is defined by substituting $\Pi_H$ into \eqref{eq:softmax_policy}:
\begin{align}
\mathcal{L}(H) &= \sum_{a=1}^k \Pi_H(a)  \sigma_a^2 .
\label{eq:functional_to_min_var}
\end{align}
\begin{remark}
	The objective considered here is the expected conditional variance
$	J(\Pi)
	=
	\mathbb E_{A\sim\Pi}
	\big[\Var(R(A)\mid A)\big].
$	This quantity measures the variability of rewards within each arm.
	An alternative risk criterion is the variance of the reward generated by
	the mixed policy,
$	\Var(R(A))$.	By the law of total variance,
\begin{align}
	\Var(R(A))
=
\mathbb E[\Var(R(A)\mid A)]
+
\Var(\mathbb E[R(A)\mid A]).
\end{align}	The second term accounts for variability between arms. The present work
	focuses on the first formulation, which naturally leads to the unbiased
	estimator \eqref{eq:rcal_reward_variance}. Nevertheless the results can be readily extended to the other case; from the numerical point of view the second case would require also an adaptation of the sampling, for instance in Algorithms \ref{alg:variance-mab} and \ref{alg:risk-aware-mab} the sampling of the arm $A_t$ will be different for each reward. 
\end{remark}

The difficulty with respect to  
 the classical MAB  \cite[section 2.8]{sutton_reinforcement_2018} lies in the
 fact that the computation of the variance cannot be obtained from a single reward observed after the arm selection. More aggregated quantities are needed. 
Assuming $\mathcal{L}(H)$ admits an expectation representation:
\begin{align}
\Lcal(H) =  
\mathbb{E}_{A \sim \Pi_H}[\Rcal(A)],
\label{eq:functional_H_rcal}
\end{align}
then the update of the preferences is done with the following formula:
\begin{equation}
H_{t+1}(a) =H_t(a) - \rho_t  g_t
, \ a=1,..., k,
\label{eq:update_Ht_pg}
\end{equation}
where
\begin{equation}
	g_t(a):=(\Rcal_t- \bar{\Rcal}_t) (\one_{a=A_t}- \Pi_{H_t}(a)),
	\label{eq:def_gt}
\end{equation}
and $\bar{\Rcal}_t$ is some baseline, usually taken as the average of previous values $\Rcal_s$, $s< t$.

It remains to be made precise what is our choice of $\Rcal$: for each arm $A_t$ selected according to $\Pi_t$ we will select two rewards $R_t$ and $R'_t$ from the distribution of the arm $A_t$ and define:
\begin{align}
    \Rcal_t = \frac{1}{2}(R_t - R_t')^2.
\label{eq:rcal_reward_variance}
\end{align}
 The \Cref{alg:variance-mab} formalizes this procedure.

\begin{algorithm}[H]
\caption{Policy gradient Variance minimizing Bandit with Mini-batch Sampling}
\begin{algorithmic}[1]
\Require number of arms $k$; batch size $\ell \in \mathbb{N}$; learning rate schedule $\{\rho_t\}_{t\ge 1}$; initialize $H_1 \in \mathbb{R}^k$, 
$\bar{\Rcal}_1=0$
\Require for each arm $a$, rewards $R_{a}$ are drawn from an unknown distribution with (unknown) mean $\mu_a$ and (unknown) variance $\sigma_a^2$
\For{$t = 1,2,\dots$}
    \State \textbf{Sample arm} $A_t \sim \Pi_{H_t}$
    \State \textbf{Draw} 2 i.i.d. rewards from arm $A_t$: $R_{t}$, $R'_{t} \sim R(A_t)$
    \State Evaluate composite reward:
    $\Rcal_t$ from \Cref{eq:rcal_reward_variance}
    \State \textbf{Update preferences (gradient descent)} with
    \Cref{eq:update_Ht_pg,eq:def_gt}
    \State Update baseline to be used in the next step:
    \begin{align}
     & \text{ future baseline: }
%    \bar{\Rcal}_{t+1} = \left(1-\frac{1}{t+1}\right)\bar{\Rcal}_{t} + \frac{1}{t+1}\cdot \Rcal_t &
    \bar{\Rcal}_{t+1} = \bar{\Rcal}_{t}+ \frac{\Rcal_t-\bar{\Rcal}_{t}}{t+1}  &
    \label{eq:baseline_var}
\end{align}
\EndFor
\end{algorithmic}
\label{alg:variance-mab}
\end{algorithm}

\noindent
We will see in \Cref{sec:theoretical_var} that under mild hypothesis the \Cref{alg:variance-mab} converges to the expected outcome (lowest variance arm).

\subsection{Risk-Aware Multi-Armed Bandit Formulation}

We discuss now a more general setting of risk-aware MAB where the 
goal is, like in the original MAB, to maximize reward but here the added requirement is to also minimize the uncertainty (variance). This leads to the definition of the
target functional in terms of the reward distribution under a policy $\Pi$ as follows: 
\begin{align}
J_r(\Pi)
=
\mathbb{E}_{A\sim\Pi}
\Big[
\lambda_\sigma\,\Var\!\big(R(A)\mid A\big)
+
\lambda_\mu\,R(A)
\Big].
\label{eq:functional_J_risk_version}
\end{align}
with parameters 
 $\lambda_\mu\le0\le \lambda_\sigma$ describing the relative importance of reward and uncertainty. The particular situation $\lambda_\mu=0$,  
$\lambda_\sigma=1$ means that the goal is to find the arm with the lowest variance.
Under the softmax parameterization, the optimization problem becomes:
\begin{equation}
\text{minimize}_{H \in \mathbb{R}^k} \; \mathcal{L}_r(H),
\label{eq:problem_mab_as_minimizationLr}
\end{equation}
where $\mathcal{L}_r(H)$ is defined by substituting $\Pi_H$ into \eqref{eq:softmax_policy}:
\begin{align}
\Lcal_r(H) &= \sum_{a=1}^k \Pi_H(a) \big( \lambda_\sigma \sigma_a^2  + \lambda_\mu \mu_a \big).
\label{eq:functional_to_min_risk}
\end{align}
Building on the \Cref{alg:variance-mab} we propose here  the 
procedure described in \Cref{alg:risk-aware-mab} adapted to the problem
in \Cref{eq:problem_mab_as_minimizationLr}.

\begin{algorithm}[H]
\caption{Policy gradient Risk-Aware Bandit with Mini-batch Sampling}
\label{alg:risk-aware-mab}
\begin{algorithmic}[1]
\Require number of arms $k$; batch size $\ell \in \mathbb{N}$; 
%learning rate $\rho>0$; 
learning rate schedule $\{\rho_t\}_{t\ge 1}$; 
initialization $H_1 \in \mathbb{R}^k$, 
$\bar{\Rcal}_1=0$
\Require for each arm $a$, rewards $R_{a}$ are drawn from an unknown distribution with (unknown) mean $\mu_a$ and (unknown) variance $\sigma_a^2$
\For{$t = 1,2,\dots$}
    \State \textbf{Sample arm} $A_t \sim \Pi_{H_t}$
    \State \textbf{Draw} $\ell$ i.i.d. rewards from arm $A_t$: $R_{t,1}, R_{t,2}, \dots, R_{t,\ell} \sim R(A_t)$
    \State \textbf{Compute statistics} on the mini-batch:
    \begin{align} 
    & \text{ empirical mean: }  
     \hat{\mu}_t := \dfrac{1}{\ell}\sum_{j=1}^{\ell} R_{t,j}   
    \label{eq:empirical_mean} &\\
    & \text{ empirical variance: } 
    \hat{\sigma}_t^2 := \dfrac{1}{\ell-1}\sum_{j=1}^{\ell} \big(R_{t,j} - \hat{\mu}_t\big)^2 
    \label{eq:empirical_var} &\\
     & \text{ composite reward: }
    \Rcal_t =  \lambda_\sigma \hat{\sigma}^2_{t}+\lambda_\mu \hat{\mu}_{t} &
    \label{eq:composite_rew_risk_version}
    \\
     & \text{ future baseline: }
    \bar{\Rcal}_{t+1} = \bar{\Rcal}_{t}+ \frac{\Rcal_t-\bar{\Rcal}_{t}}{t+1}  &
    \label{eq:baseline}
    \end{align}   
    \State \textbf{Update preferences (gradient descent)} with 
    \Cref{eq:update_Ht_pg,eq:def_gt}
\EndFor
\end{algorithmic}
\end{algorithm}

\noindent
The theoretical properties of both algorithms are analyzed in the next section.

\section{Theoretical convergence results} \label{sec:cv_proof}

We first prove that the update formula \eqref{eq:update_Ht_pg} is a stochastic algorithm {\it \`a la} Robins and Monro \cite{robbins_stochastic_1951}. Denote $\Fcal_t$ the filtration constructed with information available up to time step $t$. Note that $A_t$ is independent of $\Fcal_t$.

\begin{lemma} For both \Cref{alg:variance-mab,alg:risk-aware-mab} 
we have:
\begin{equation}
	\E[\one_{a=A_t}|\Fcal_t] = \Pi_{H_t}(a).    
	\label{eq:filtration_ft_at}
\end{equation}
Moreover for \Cref{alg:variance-mab}
\begin{equation}
\E    \left[ \left. g_t
\right|\Fcal_t \right] =   \left. \nabla_H \Lcal(H)\right|_{H=H_t},
\label{eq:non_biased_gradient_var}
\end{equation}
while for
\Cref{alg:risk-aware-mab}
\begin{equation}
\E    \left[ \left. g_t
\right|\Fcal_t \right] =   \left. \nabla_H \Lcal_r(H)\right|_{H=H_t},
\label{eq:non_biased_gradient_risk}
\end{equation}
which means that $g_t$ in update \Cref{eq:update_Ht_pg} is indeed an un-biased estimation of the true gradient of $\nabla_H \Lcal(H)$ ( \Cref{alg:variance-mab}) or
 $\nabla_H \Lcal_r(H)$ ( \Cref{alg:risk-aware-mab}). 
\label{lemma:unbiased_grad}    
\end{lemma}
\begin{proof}
To prove \Cref{eq:filtration_ft_at} it is enough to remark that since $A_t$ is independent of $\Fcal_t$ we have 
\begin{align}
\E[\one_{a=A_t}|\Fcal_t] = \E[\one_{a=A_t}]= \Pbb[A_t]= \Pi_{H_t}(a).    
\end{align}
\noindent Note that assertion 
\Cref{eq:non_biased_gradient_var} is a particular case of 
\Cref{eq:non_biased_gradient_risk} for $\lambda_\mu=0$, $\lambda_\sigma=1$ and 
$\ell=2$. It remains to prove \Cref{eq:non_biased_gradient_risk}.
First we compute the term involving $\bar{\Rcal}_t$:
\begin{align}& 
\E[ \bar{\Rcal}_t (\one_{a=A_t}- \Pi_{H_t}(a))|\Fcal_t] = 
\bar{\Rcal}_t \E[ \one_{a=A_t}- \Pi_{H_t}(a)|\Fcal_t] 
\nonumber \\&
\bar{\Rcal}_t  \left(  \E[ \one_{a=A_t}|\Fcal_t] 
- \Pi_{H_t}(a) \right) = 0.
\label{eq:rtbar_proof_nobias}
\end{align}
Denote now $\Fcal_t^+$ the filtration $\Fcal_t$ enlarged with $A_t$. Then
with the notations 
\Cref{eq:empirical_mean}  and \Cref{eq:empirical_var} 
we have 
\begin{align}& 
\E[ {\Rcal}_t \one_{a=A_t}|\Fcal_t] = 
\E[ \E[ {\Rcal}_t \one_{a=A_t}|\Fcal_t^+]|\Fcal_t]
=\E[ (\lambda_\sigma {\sigma}_a^2 + \lambda_\mu {\mu}_a ) \one_{a=A_t}|\Fcal_t]
\nonumber \\ &
=(\lambda_\sigma {\sigma}_a^2 + \lambda_\mu {\mu}_a ) \Pi_{H_t}(a).
\label{eq:rt1a_proof_nobias}
\end{align}
On the other hand 
\begin{align}& 
\E[ {\Rcal}_t \Pi_{H_t}(a)|\Fcal_t] = 
\Pi_{H_t}(a)\E[ {\Rcal}_t |\Fcal_t] = 
\Pi_{H_t}(a) \sum_b \E[ {\Rcal}_t \one_{b=A_t}|\Fcal_t]
\nonumber \\ &
= \sum_b (\lambda_\sigma {\sigma}_b^2 + \lambda_\mu {\mu}_b ) \Pi_{H_t}(b)\Pi_{H_t}(a) .
\label{eq:rtpi_proof_nobias}
\end{align}
where for the last equality we used the same technique as in the previous one.
Using the derivation rule for the softmax:
\begin{align}
\nabla_{H_b} \Pi_H(a) = \Pi_H(a) (  \one_{b=a} - \Pi_H(b))  
\end{align}
 together with
\Cref{eq:rtbar_proof_nobias,eq:rt1a_proof_nobias,eq:rtpi_proof_nobias}
we obtain
\begin{align}
&    \E[ g_t(a)|\Fcal_t] =
  \E[ (\Rcal_t- \bar{\Rcal}_t) (\one_{a=A_t}- \Pi_{H_t}(a))|\Fcal_t] 
\nonumber \\ &
= \sum_b \Pi_{H_t}(b)(\lambda_\sigma {\sigma}_b^2 + \lambda_\mu {\mu}_b )\cdot (  \one_{b=a} - \Pi_{H_t}(a))  
\nonumber \\ &
= \sum_b \nabla_{H_a} \Pi_{H_t}(b) (\lambda_\sigma {\sigma}_b^2 + \lambda_\mu {\mu}_b )
= \left. \nabla_H \Lcal_r(H)\right|_{H=H_t},
\end{align}
which is the conclusion of the Lemma.\end{proof}
We can now state the main result for the \Cref{alg:risk-aware-mab}.

\begin{proposition}
Assume that $\rho_t=\rho$ (constant) and that all outcomes are unique i.e.,:
\begin{align}
    \forall a\neq b, a,b \le k \ : 
    \lambda_\sigma {\sigma}_b^2 + \lambda_\mu {\mu}_b \neq 
    \lambda_\sigma {\sigma}_a^2 + \lambda_\mu {\mu}_a.
    \label{eq:different_qa_qb_hyp}
\end{align}
Denote $a^\star \le k$ the unique index that minimizes 
$ \lambda_\sigma {\sigma}_b^2 + \lambda_\mu {\mu}_b $ among all $b \le k$.
Then
\begin{enumerate}
    \item The  Dirac mass in $a^\star$ denoted $\Pi^\star$ is the unique distribution that minimizes $\Pi \mapsto J_r(\Pi)$.
\item
Assume moreover that the rewards are bounded i.e.,:
\begin{align}
\text{there exists } R_{max} < \infty \text{ such that } \forall a \le k : |R(a)| \le R_{max} \text{ as random variable}. 
\label{eq:Ra_bounded}
\end{align}
Then the distribution $\Pi_{H_t}$ in \Cref{alg:risk-aware-mab} 
converges to $\Pi^\star$ almost surely, i.e.:
\begin{align}
\lim_{t\to \infty}\Pi_{H_t}(a^\star)=1 \ \ a.s. \text{ and } \lim_{t \to \infty }\Pi_{H_t} = \Pi^\star \ a.s.
\label{eq:convergence_risk}
\end{align}
\end{enumerate}
\label{prop:cv_risk}
\end{proposition}
\begin{proof}
The first item is an obvious consequence of assumption \eqref{eq:different_qa_qb_hyp}.
For the second item we invoke a recent result by Mei et al., in \cite{mei2023stochastic}. With their definition we denote $q_\star(a) = -(\lambda_\sigma {\sigma}_a^2 + \lambda_\mu {\mu}_a)$, for $a\le k$. Note that the sample $\Rcal_t$ is bounded because all $R(a)$ are bounded. Also note that from \Cref{lemma:unbiased_grad} the \Cref{alg:risk-aware-mab} is a softmax policy gradient in the notations of  Mei et al., \cite{mei2023stochastic}. Then using their Theorem 5.1 we obtain the conclusion.
\end{proof}

\begin{remark}
The boundedness assumption may seem somehow strong and indeed it is. However in practice this is not an issue because in all reasonable applications excessively high values of the reward will cause instabilities in the model and are not realistic. It can also be proved that, even if distributions $R(a)$ are not bounded but satisfy  \Cref{eq:different_qa_qb_hyp} there exists a $R_{max}$ such that truncating the distributions $R(a)$ to $R_{max}$ will conserve both the relation \eqref{eq:different_qa_qb_hyp} and the optimal $a^\star$; it is enough then to run the algorithm by truncating the rewards to $R_{max}$ to obtain convergence to the optimal solution. On the other hand, for applications where large values of the reward are critical (heavy-tailed distributions) this assumption may be numerically fragile and it would be desirable to explore alternative proofs, which we leave as future work.
\end{remark}

\subsection{Variance minimization version}\label{sec:theoretical_var}

As a consequence of \Cref{prop:cv_risk} we obtain:
\begin{proposition}
Assume \eqref{eq:Ra_bounded} and that
\begin{align}
    \forall a\neq b, a,b \le k \ : 
    {\sigma}_b\neq {\sigma}_a.
    \label{eq:different_sigmaa_sigmab_hyp}
\end{align}
Then convergence \eqref{eq:convergence_risk} also holds for \Cref{alg:variance-mab} where $\Pi^\star$ is the Dirac mass supported in the $a^\star$ that minimizes the variances of $R(a)$.
\end{proposition}
\begin{proof}
The proof is a specialization of the previous result for the case $\lambda_\mu=0$, $\lambda_\sigma=1$, $\ell=2$.    
\end{proof}

\section{Numerical simulations} \label{sec:numerical}

The Python implementation is available on Github~\footnote{\scriptsize\url{https://github.com/gabriel-turinici/min_variance_and_risk_averse_softmax_MAB} version April 21st 2026.}.

\subsection{Toy problem: 2 and 10 arms, centered, well separated}
\label{sec:toy_2arms}

We first test a toy problem: $k=2$ arms, reward distributions centered, one of unit variance and the other variance $2$: $\mu=(0,0)$, $\sigma=(1,2)$. We run for $200$ time steps, $\rho_t=0.5$ and take $1000$ realizations. The results are given in \Cref{fig:toy2arms_minvar} and show good convergence of the regret (to zero) and the identification of the optimal arm (frequency of the optimal arm goes to $100\%$).

\begin{figure*}
\begin{center}
\includegraphics[width=.49\linewidth]{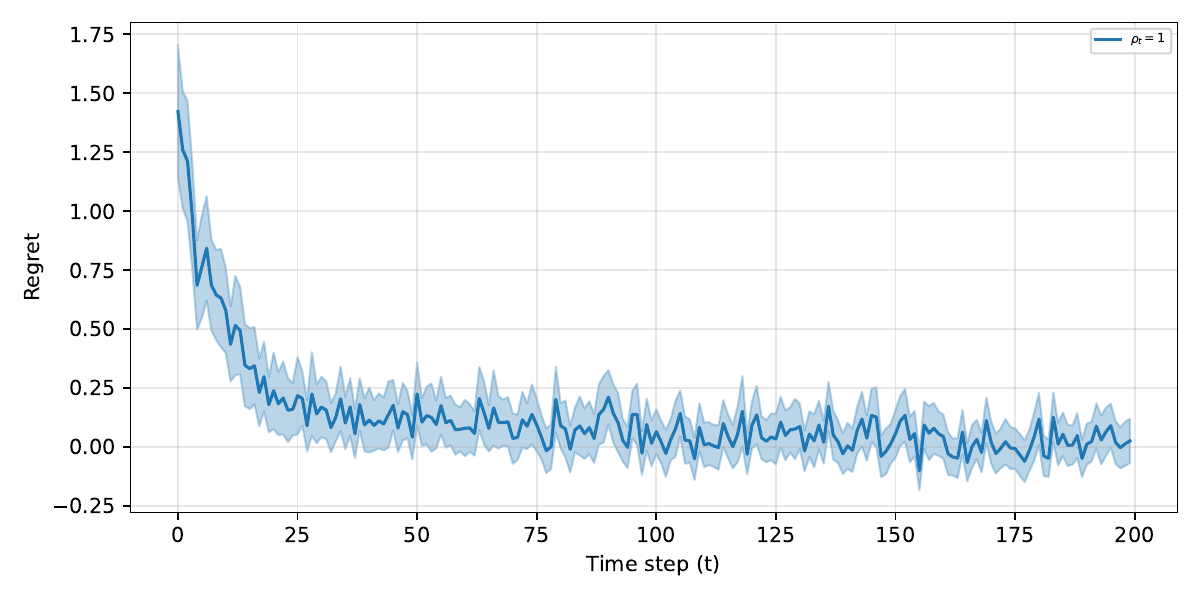}
\includegraphics[width=.49\linewidth]{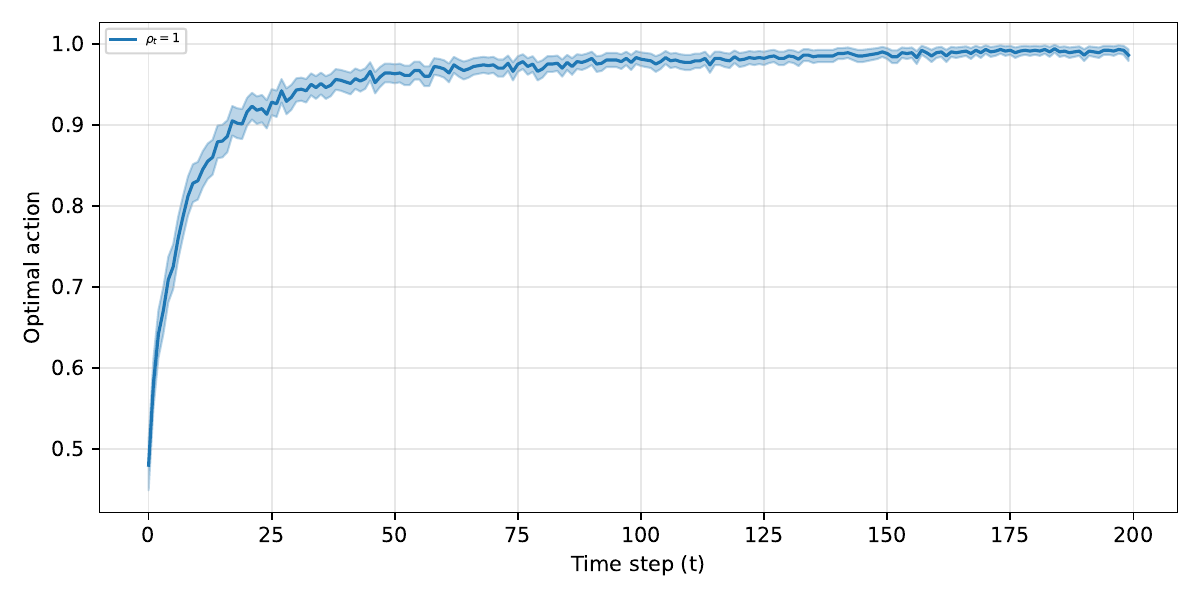}
\caption{
The toy example in \Cref{sec:toy_2arms} for $k=2$ arms.
The average regret (left plot) and average optimal action frequency (right plot); we also display the 95\% Confidence interval for each. Run details: 
 $H_1=(0,0)$ (uniform), learning rate $\rho_t=0.5$, $200$ time steps.
}  \label{fig:toy2arms_minvar}
\end{center}
\end{figure*}

We tested next a similar case with only difference that there are $k=10$ arms, all centered with $\sigma_a=2$ $\forall a < k$ except for the last arm where $\sigma_{k}=1$.
The results are given in \Cref{fig:toy10arms_minvar}; again, we have good convergence of the regret (to zero) and the identification of the optimal arm (likeliness of selecting the optimal arm goes to $100\%$). Note that this case does not strictly satisfy assumption \eqref{eq:different_sigmaa_sigmab_hyp} but in fact it is important that the best arm be separated from the others, the non-optimal arms can have similar variances.

\begin{figure*}
\begin{center}
\includegraphics[width=.45\linewidth]{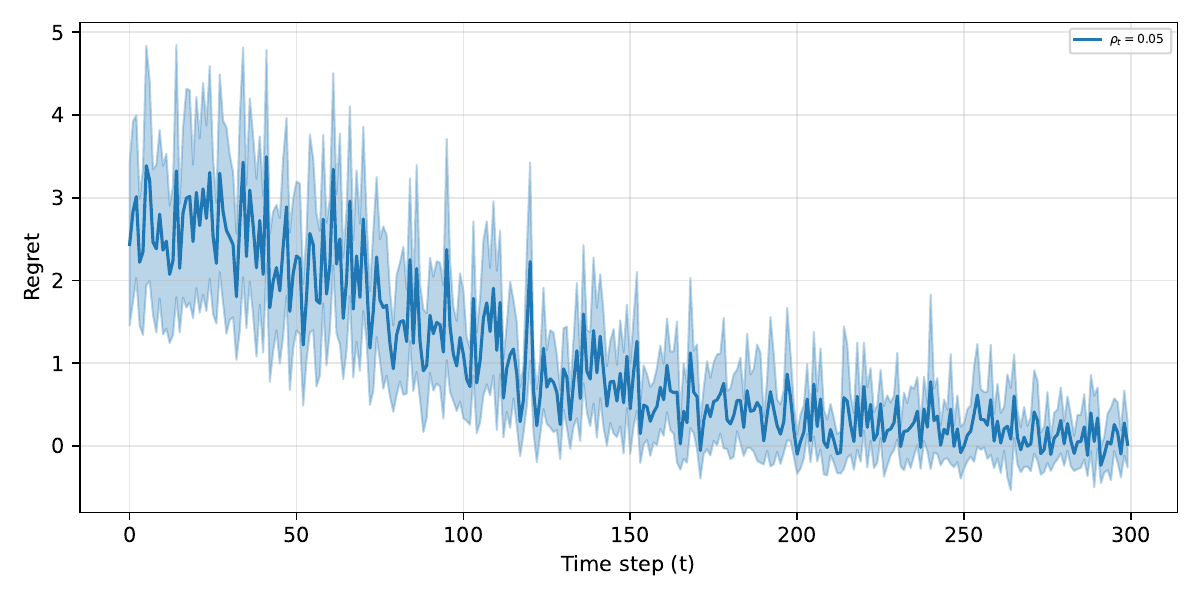}
\includegraphics[width=.45\linewidth]{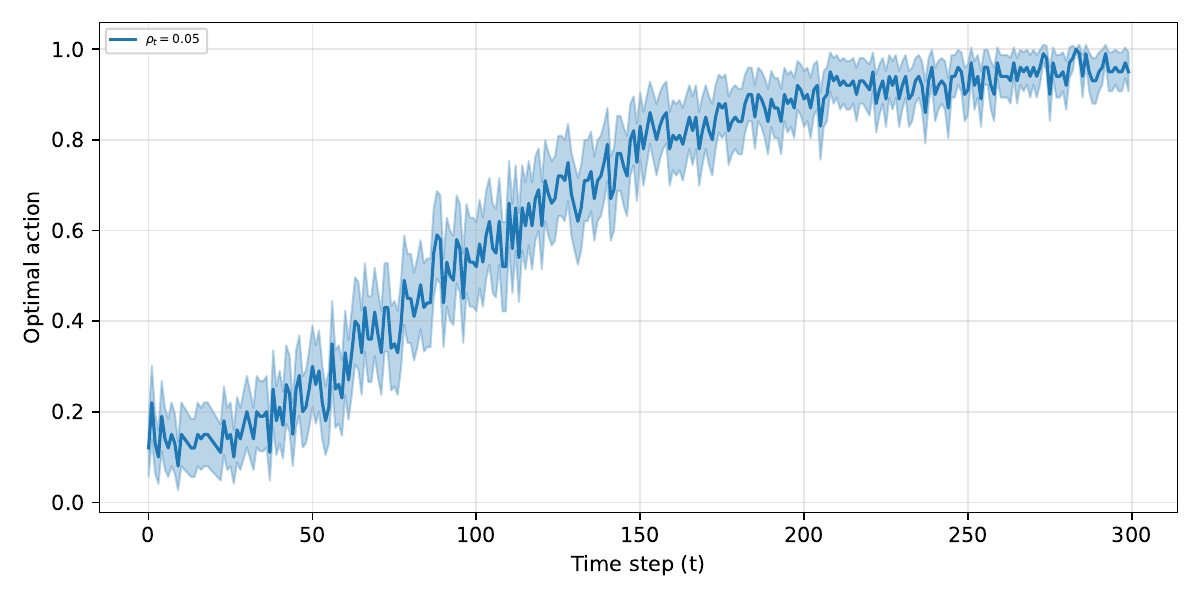}
\caption{
The toy example in \Cref{sec:toy_2arms} for $k=10$ arms.
The average regret (left plot) and average optimal action frequency (right plot); we also display the 95\% Confidence interval for each. Run details: 
 $H_1=(0,0)$ (uniform), learning rate $\rho_t=0.05$,  $300$ time steps.
}  \label{fig:toy10arms_minvar}
\end{center}
\end{figure*}

\subsection{Difficult setting: 10 arms, random separation}
\label{sec:difficult10}

We consider now a more difficult case where $k=10$ but the separation between the optimal arm and the others can be very small; we perform 
$M=1000$ tests of $2000$ steps each; for each of the $M$ tests we sample, as in \cite[fig 2.5 page 38]{sutton_reinforcement_2018}
$k=10$ arms with 
$\mu_a$, $a=1, ..., k$ independent and normally distributed with mean $4$; the variances $\sigma_a^2$ are sampled uniformly in the interval $[1,5]$. Note that there is nothing that prevents two variances to be very close and in fact this will happen often enough.

The average regret over the $M=1000$ runs are presented in 
\Cref{fig:regret_rho01u_minvar} and the average optimal action frequency in
\Cref{fig:regret_rho01u_optact}. The average regret appears very satisfactory and the optimal action frequency above 70\% which is a very good figure; going to 100\% is not realistic because in many cases the difference between the lowest variance arm and the second lowest is tiny and it is not possible to reliably identify the lowest one with enough confidence; however this has no impact on the optimality of the reward, cf. the good quality of the regret figures.

\begin{figure*}
\begin{center}
\includegraphics[width=.75\linewidth]{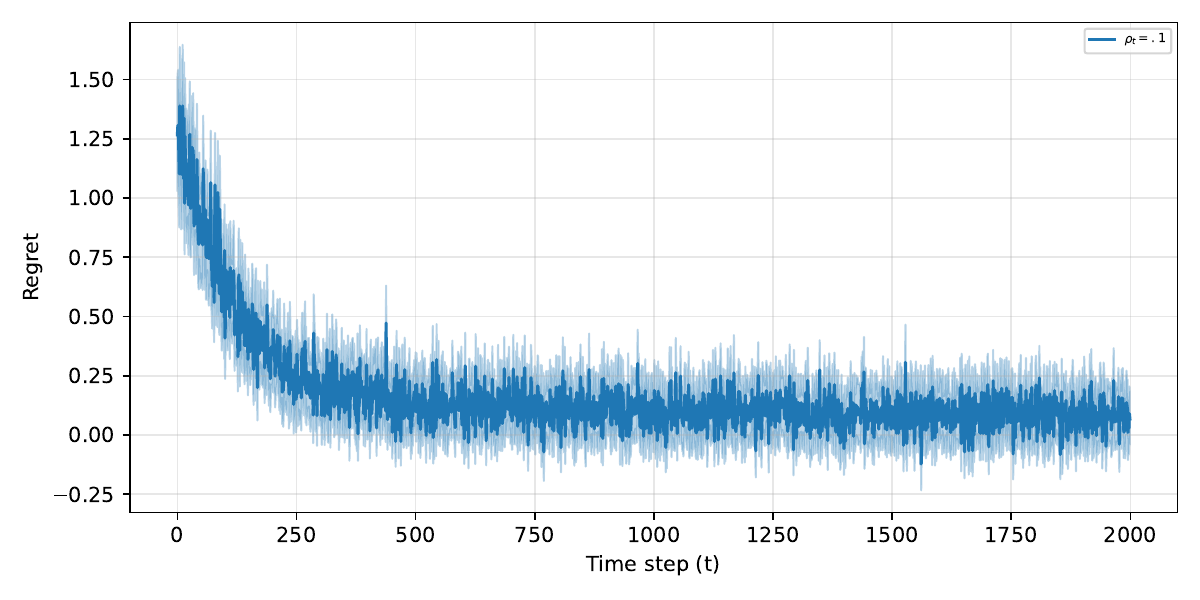}
\caption{ Test case in \Cref{sec:difficult10}, 
the average regret and its 95\% Confidence interval;
 $H_1=(0,...,0)$ (uniform) and learning rate $\rho_t=0.1$.
}  \label{fig:regret_rho01u_minvar}
\end{center}
\end{figure*}

\begin{figure*}
\begin{center}
\includegraphics[width=.75\linewidth]{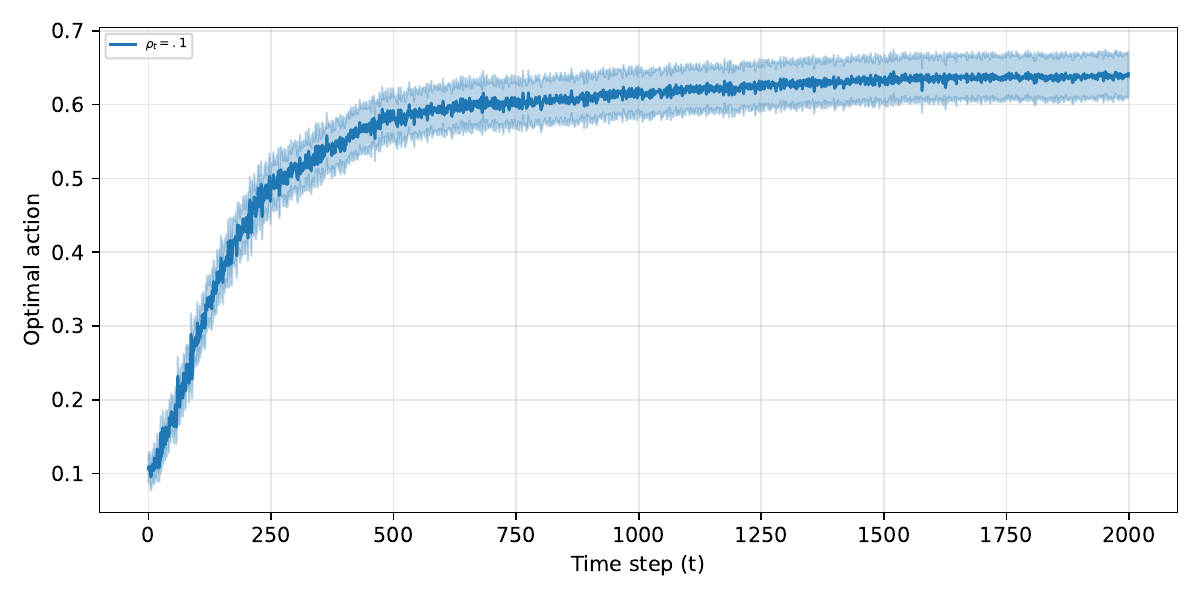}
\caption{ Test case in \Cref{sec:difficult10}, 
the average number of optimal arm selection and its 95\% Confidence interval;
 $H_1=(0,...,0)$ (uniform) and learning rate $\rho_t=0.1$.
}  \label{fig:regret_rho01u_optact}
\end{center}
\end{figure*}

\section{Summary and discussion} \label{sec:conclusion}

We consider a risk-aware formulation of the multi-armed bandit problem where the goal is to find the minimal variance arm and not the maximal reward as the classical objective. This goal is more difficult because errors in tracking the averages may result in large errors in the variance estimate. 
By adopting a softmax policy parameterization, we developed an algorithm that constructs an unbiased estimator of the variance through paired sampling and we established its convergence under mild assumptions. Empirical evaluations demonstrate the practical effectiveness of the proposed method. Beyond variance minimization, the framework accommodates broader risk-sensitive objectives, enabling a principled balance between reward and stability in sequential decision-making.

Several directions merit further investigation. First, extending the proposed approach to contextual bandits, time-dependent bandits and reinforcement learning settings would broaden its applicability to dynamic environments. Second, exploring va\-ri\-ance reduction techniques may improve sample efficiency in high-dimensional action spaces. Finally, theoretical analysis of convergence for variance-reduction settings could provide a more rigorous understanding of algorithmic stability and provide performance guarantees.

\bibliographystyle{splncs04}
\bibliography{refs}

\end{document}